\documentclass[letterpaper, 10 pt, conference]{ieeeconf}

\IEEEoverridecommandlockouts
\usepackage{cite}
\usepackage[bookmarks=true]{hyperref}
\usepackage{color}
\usepackage{times}
\usepackage{epsfig}
\usepackage{epstopdf}
\usepackage{graphicx}
\usepackage{amsmath}
\usepackage{amssymb}
\usepackage{bbm}
\usepackage[ruled,vlined,linesnumbered]{algorithm2e}
\usepackage{graphicx}
\usepackage{verbatim}
\usepackage{booktabs}
\usepackage{multirow}
\usepackage{makecell}
\usepackage{subfigure}
\usepackage{colortbl,booktabs}
\usepackage{lscape}
\usepackage{url}
\usepackage{makecell}
\usepackage{xcolor,colortbl}

\newtheorem{proposition}{Proposition}

\title{\LARGE \bf Representing Multi-Robot Structure through Multimodal Graph Embedding for the Selection of Robot Teams}

\author{
    Brian Reily$^{1}$, Christopher Reardon$^{2}$, and Hao Zhang$^{1}$%
    \thanks{*This work was partially supported by ARL DCIST CRA W911NF-17-2-0181, ARO W911NF-17-1-0447, and NSF CNS-1823245.}%
    \thanks{$^{1}$Brian Reily and Hao Zhang are with the Human-Centered Robotics Laboratory
    at the Colorado School of Mines, Golden, CO, 80401, USA.
    email: \{breily, hzhang\}@mines.edu}%
    \thanks{$^{2}$Christopher Reardon is with the U.S. Army Research Laboratory in Adelphi, MD, 20783, USA.
    email: christopher.m.reardon3.civ@mail.mil}%
}

\begin{document}
\maketitle
\thispagestyle{empty}
\pagestyle{empty}

%%%%
\begin{abstract}
%%%%

Multi-robot systems of increasing size and complexity
are used to solve large-scale problems, such as area exploration and search and rescue.
A key decision in human-robot teaming is dividing a multi-robot system into teams to address
separate issues or to accomplish a task over a large area.
In order to address the problem of selecting teams in a multi-robot system,
we propose a new multimodal graph embedding method
to construct a unified representation
that fuses multiple information modalities to describe and divide a multi-robot system.
The relationship modalities are encoded as directed graphs that can encode
asymmetrical relationships,
which are embedded into a unified representation for each robot.
Then, the constructed multimodal representation is used to determine teams
based upon unsupervised learning.
We perform experiments to evaluate our approach on
expert-defined team formations,
large-scale simulated multi-robot systems, and a system of physical robots.
Experimental results show that
our method successfully decides correct teams based on
the multifaceted internal structures describing multi-robot systems,
and outperforms baseline methods based upon only one mode of information,
as well as other graph embedding-based division methods.

%%%%
\end{abstract}
%%%%

%%%%
\section{Introduction}
%%%%

Because of their robustness and flexibility,
multi-robot systems are being increasingly researched and used in large-scale applications,
such as disaster response \cite{erdelj2017wireless},
search and rescue \cite{kruijff2012experience},
and area exploration \cite{bayat2017environmental}.
Accomplishing complex operations using multi-robot systems typically
requires the division of a system into effective teams that are capable of achieving multiple
separate tasks simultaneously
or accomplishing mission objectives over a large
area \cite{khamis2015multi}.
However, as the number of robots in a system increases,
it becomes
cognitively more difficult for humans to understand and command
\cite{bales2017neurophysiological}.

\begin{figure}[t]
    \centering
    \includegraphics[width=0.47\textwidth]{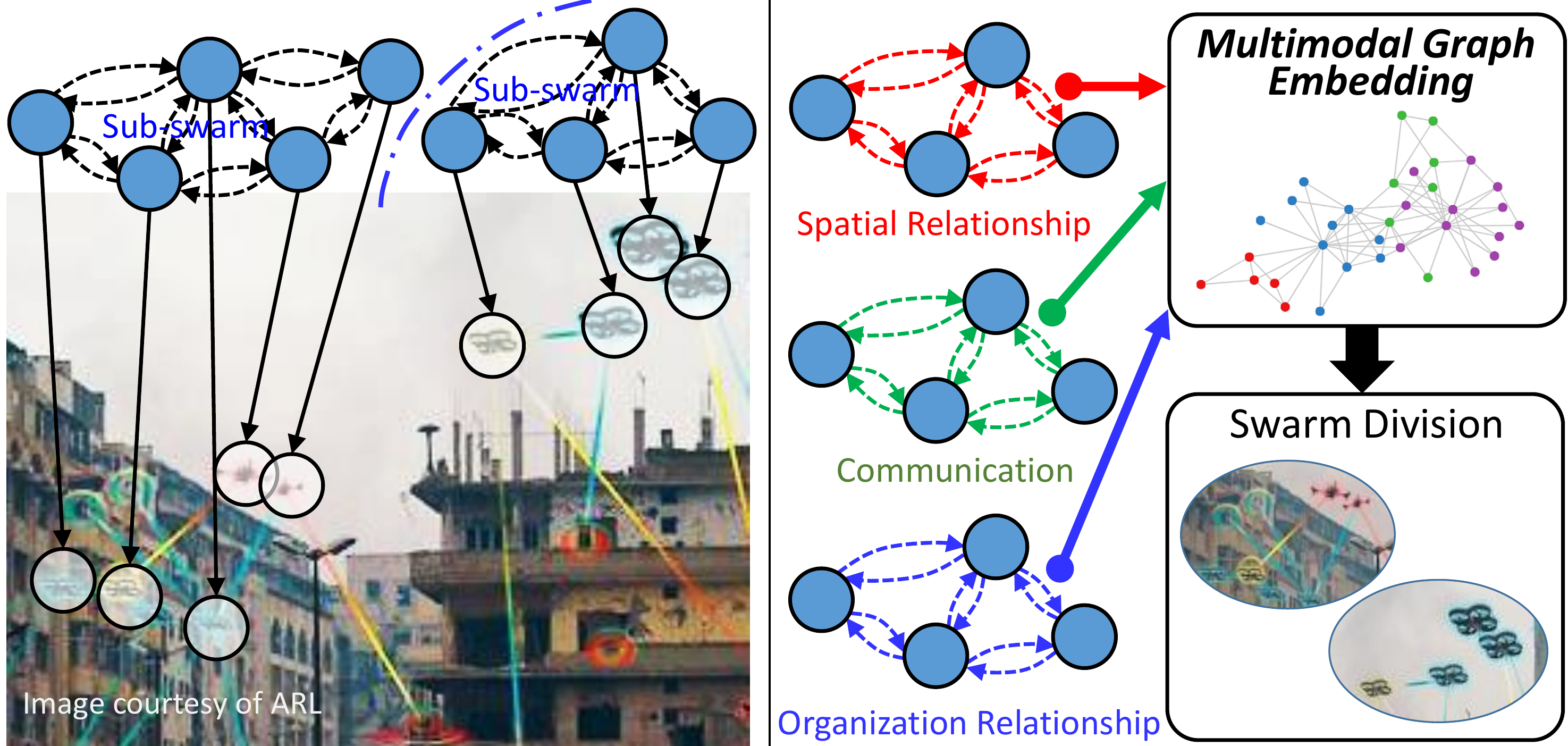}
    \vspace{-8pt}
    \caption{A motivating example of automatic multi-robot division
    and our solution based on multimodal graph embedding.
    In real-world multi-robot systems, robots typically have a variety of relationships,
    such as spatial relationships,
    communication connectivity, and organization hierarchy.
    Our proposed multimodal graph embedding approach
    can integrate multiple relationship graphs and identify effective multi-robot teams.
    }
    \label{fig:overview}
\end{figure}

Successful division of a large multi-robot system into effective teams
is a
particularly difficult problem as the size and complexity
of the interaction between robots increases.
At scale, the complexities of internal relationships become difficult
for human operators to conceptualize and integrate \cite{velagapudi2008scaling}.
Figure \ref{fig:overview} provides an illustration of how robots can appear organized in physical space,
but also contain hierarchical relationships
or communication capabilities within a system that are more difficult to
perceive.
These relationships are further complicated by robot interactions with obstacles and the
surrounding environment.
When combined, these challenges result in robotic system states that are both difficult for
a human operator to accurately perceive and
for a system to display \cite{lewis2010teams}.
Thus, effective representations of multi-robot structure are critical for
understanding and selecting team divisions.

%%%%%
Due to the importance of multi-robot teams, division approaches have
been widely studied in the literature.
A commonly used paradigm approached the division of a multi-robot
system based solely on spatial locations of members,
assigning territories to each
robot \cite{schneider1998territorial,rekleitis2004limited}.
These methods have the drawbacks of relying on spatial
information only and assigning agents to individual territories,
without the ability to identify teams or grouped divisions.
Another paradigm views multi-robot division as a task
allocation problem, e.g., based on
task sequencing \cite{castillo2014temporal,brutschy2014self},
game theory \cite{jang2018anonymous,jang2018comparative},
and Markov chains \cite{jang2017bio}.
These techniques require
explicit knowledge about the tasks,
such as task sequences, requirements, or priorities,
as well as knowledge of opponent behaviors,
which may not be available to a human operator.
In addition,
methods were also implemented
that are inspired by biological swarms such as
ants \cite{wu2018modeling,labella2006division}
or wasps \cite{bonabeau1997adaptive}.
These techniques often assume simple agents,
often without the ability to communicate,
and are not capable of creating a complete representation of
multi-robot structure by addressing multiple relationship modalities.

% high-level description of proposed approach
To address the selection of multi-robot teams,
we propose a novel multimodal graph embedding approach
to encode diverse relationships of multiple robots as graphs
and integrate these multiple graphs
into a unified representation that is applied to
divide a multi-robot system into teams.
We model each internal relationship of the robots
using a directed graph
as an information modality.
Given a set of member relationships,
we construct multiple graphs that are applied as the input to our approach.
Then, we propose a new multimodal Katz index to integrate
multiple graphs of robot relationships and embed them into a
unified representation for each robot.
Then, the constructed representation is used to identify multi-robot teams
through unsupervised learning.
Our multimodal graph-embedded robot division approach is capable of
fusing diverse robotic relationships
and identifying team divisions without requiring explicit knowledge of tasks.

% contributions
This paper has two major contributions.
\begin{enumerate}
\item First, we propose a novel multimodal graph embedding approach
that formulates robot team selection as a graph-based embedding and clustering problem,
models multiple asymmetrical relationships of robots using directed graphs,
and fuses them into a unified representation to identify robotic teams.
\item Second, we perform a comprehensive evaluation of the proposed method in
the scenarios using expert-defined team formations,
large-scale simulated multi-robot systems, and
teams of physical robots.
We validate that our multimodal graph embedding approach is effective in identifying
robotic teams, and obtains superior performance over existing graph-based division methods.
\end{enumerate}

%%%%
\section{Related Work} \label{sec:related}
%%%%

In this section, we review previous approaches to multi-robot division,
representation of multi-robot systems as graphs,
and existing graph embedding methods.

\subsection{Multi-Robot Division}

Dividing multi-robot systems to accomplish
multiple objectives simultaneously
has been approached previously from three main perspectives:
spatial-based division, task allocation, and biologically inspired
methods.

Spatial-based division was implemented by assigning robots to their
own territories \cite{schneider1998territorial,rekleitis2004limited}.
It was typically formulated as an optimization problem, where a group of robots
diffuses to optimize coverage of an area \cite{li2017decentralized}.
This spatial-based perspective has two main drawbacks.
First, it relies solely on the spatial locations of robots, omitting
other information available that could be utilized to better divide
the system.
Second, it assigns members to areas individually, and is unable to divide
a multi-robot system into teams, each still capable of working together.

Task allocation was implemented for multi-robot division in a variety of ways.
Explicit assignment was designed based on sub-goals \cite{ephrati1994divide},
or the temporal relationships between tasks % (i.e., if they must occur in a set sequence)
\cite{castillo2014temporal,brutschy2014self}.
Markov chains were also developed to schedule task
assignments \cite{jang2017bio}.
Agents were designed to negotiate with each other for
sub-tasks \cite{zlot2002multi}, and game theory was applied to
find optimal task assignments \cite{jang2018anonymous,jang2018comparative}.
Multi-robot division based on task allocation has the major drawback
of requiring explicit information about the task.
This can range from requiring task dependency information in order
to schedule sequences, task composition in order to assign
sub-tasks, or task priorities and requirements in order to determine
the best robot to assign the task to.
Many approaches also assign tasks to individual robots, as opposed to
tasking groups to accomplish a single task.

Methods inspired by biological swarms were also designed for
multi-robot division, inspired by how insects divide into sub-groups.
Methods were introduced based on ant foraging \cite{labella2006division}
and division of labor in wasp colonies \cite{bonabeau1997adaptive}
and ant colonies \cite{wu2018modeling}.
Biological approaches generally
assume swarm agents are extremely limited in their ability to
sense, communicate, and process information.
Previous bio-inspired methods cannot incorporate multimodal relationships of swarm members.

\subsection{Graph Representations of Multi-Agent Systems}

Many previous works used graphs to represent multi-agent systems,
by representing agents in the team with
vertices and relationships between agents as edges.
Control laws were first applied as edges between agents, where these
laws controlled spacing between agents and were applied to construct group
trajectories around obstacles \cite{desai2001modeling,desai2001modeling2,desai2002graph}.
Graphs that described multiple nearest neighbors through control laws
were applied in \cite{li2005stability} and \cite{gennaro2005formation}.
Communication channels between robots were represented graphically in \cite{fierro2002modular}.
Graph rigidity was used in
\cite{olfati2002graph,olfati2003cooperative} in
an optimization-based method to perform split and
rejoin maneuvers around obstacles.

Methods to find communities in graphs are mainly based on analyzing edges between vertices.
Cut-based methods cut the graph to form the best two
clusters, and include min-max cuts \cite{ding2001min} and
normalized cuts \cite{shi2000normalized}.
Communities based on the probability of new links were used in \cite{newman2001clustering},
later extended with the idea of modularity \cite{newman2004finding,newman2006modularity},
which constructs a segmentation metric and uses this to create divisions.
%Similar to this, \cite{xu2007scan} introduced a structural clustering algorithm
%to not just identify communities but to additionally distinguish hubs and %outlier nodes.
These community-finding algorithms are unable to operate on
multiple graphs, and often cannot divide a graph into an arbitrary
number of communities but instead rely on the structure
of the graph
to define the number of communities.

\subsection{Graph Embedding}

Graph embedding is the representation of graph structure in vector space,
allowing typical machine learning techniques to be applied
\cite{goyal2018graph,cai2018comprehensive,cui2018survey}.
Techniques including node2vec \cite{grover2016node2vec} employ random walks,
representing nodes based upon their transition probabilities to and from
other nodes.
%(i.e., nodes would have similar representations if transitions from each commonly went to the same set of nodes).
%Deep learning was applied in SDNE \cite{wang2016structural},
%which uses auto-encoders to learn a node representation.
Methods based on matrix descriptions of graphs, such as adjacency matrices and
similarity matrices, include structure preserving embedding \cite{shaw2009structure},
graph factorization \cite{ahmed2013distributed},
and High-Order Proximity preserved Embedding (HOPE) \cite{ou2016asymmetric}.
Most graph embedding techniques focused on single graphs
representation, e.g., to encode social networks or academic citation networks,
without the ability to integrate multiple relationships.
The effectiveness of graph embedding for multi-robot division has also not
well studied in the literature.

%%%%
\section{Approach}
\label{sec:approach}
%%%%

We propose a new multimodal graph embedding approach to integrate multiple
directed graphs
that encode relationships in a multi-robot system
into a unified representation
that is used to identify multi-robot teams.

\subsection{Multimodal Robotic Structure Embedding}

Given a specified relationship (e.g., communication connectivity)
of a multi-robot system including $N$ members,
we represent a relationship among robots as a
graph $\mathcal{G} = ( \mathcal{V}, \mathcal{E} )$, where
$\mathcal{V} = \{ v_1, \dots, v_N \}$ denotes the
set of vertices, each corresponding to a robot,
and $\mathcal{E}$ denotes the
set of directed edges between these vertices.
The direction and weight of each edge $e_{ij} = ( v_i, v_j ) \in \mathcal{E}$
depends on the type of relation the graph is representing.

In real-world deployment,
robotic members within a system typically have multiple various relationships
(e.g., spatial relationships, communication connectivity, and organization hierarchy).
When multiple robotic relationships are available,
we encode the system with $M$ graphs, where $\mathcal{G}_m$ is the graph
describing the $m$-th relationship between robots.
Each graph $\mathcal{G}_m$ is described by an adjacency matrix
$\mathbf{A}_m \in {\mathcal{R}}^{N \times N}$, where each element
$a_{ij}$ is the weight of the edge connecting vertex $v_i$ to
vertex $v_j$.
The proposed method is able to represent both directed (e.g. communication from one
member to others) and undirected (e.g., spatial distances) relationships.
In the case of representing undirected relationships, the weight
of edges satisfies $a_{ij} = a_{ji}$.

To represent graphs in vector space (i.e., graph embedding),
the Katz index \cite{katz1953new} has shown to be a promising method
that has been used to address real-world graph problems.
The Katz index represents the similarity of a pair of vertices given a graph,
by summing the weights of edges along all paths between the two vertices
to represent how closely connected the two vertices are.
Given a decay parameter
$\alpha>0$ that weights paths based on their length $L$, the Katz index can produce
a similarity matrix $\mathbf{S}$ to describe a graph:
\begin{equation}
\label{eq:katz}
\mathbf{S} = \sum_{l = 1}^{L} \alpha^{l} \mathbf{A}^{l}
\end{equation}

This equation can be rewritten as
\begin{equation}
\label{eq:katz2}
\mathbf{S} = {( \mathbf{I} - \alpha \mathbf{A} )}^{-1} - \mathbf{I}
\end{equation}
where $\mathbf{I}$ is the identity matrix.

To achieve our objective of encoding multiple graphs
and embedding them into a single vector representation,
we propose a new multimodal Katz index
that is able to take multiple graphs as the multimodal input modalities
and form a single similarity matrix
$\mathbf{S} \in {\mathcal{R}}^{N \times N}$ that integrates the information
of all graphs.
To do this,
we implement an adjacency matrix $\mathbf{A}_m$ for each graph $\mathcal{G}_m$,
where $m = 1, \dots, M$, and introduce
a weight $w_m$ for each $\mathbf{A}_m$ describing the importance
of the relationship encoded by $\mathcal{G}_m$, where
$1 = \sum_{m=1}^{M} w_m$.
Values of $w_m$ can be specified by human experts to incorporate prior knowledge
and preferences, or estimated by hyperparameter selection methods (e.g., grid search).
We then construct the multimodal similarity matrix $\mathbf{S}$
that embeds information of all graphs as follows:
\begin{equation}
\label{eq:katz_weighted}
\mathbf{S} = {\left( \mathbf{I} - \alpha \sum_{m = 1}^{M} w_m \mathbf{A}_m \right)}^{-1} - \mathbf{I}
\end{equation}

\begin{proposition}\label{proposition}
$\left( \mathbf{I} - \alpha \sum_{m=1}^M w_m \mathbf{A}_m \right)$ is invertible
when  $\alpha^{-1}$ is greater than the largest eigenvalue of $\mathbf{A}^{*}$, where $\mathbf{A}^{*} = \sum_{m=1}^M w_m \mathbf{A}_m$.
\end{proposition}
\begin{proof}
For $\left( \mathbf{I} - \alpha \mathbf{A}^{*} \right)$ to not be invertible,
$\operatorname{det} \left(\mathbf{I} - \alpha \mathbf{A}^{*} \right) = 0$.
This is equivalent to $\operatorname{det} \left( \mathbf{A}^{*} - \alpha^{-1}\mathbf{I}\right) = 0$.
This determinant is only zero if $\alpha^{-1}$ is equal to the largest
eigenvalue of $\mathbf{A}^{*}$, denoted as $e_1$.
When $\alpha < \frac{1}{e_1}$, the inverse is valid.
\end{proof}

\begin{algorithm}[tb]
\SetAlgoLined
\SetKwInOut{Input}{Input}
\SetKwInOut{Output}{Output}
\SetNlSty{textrm}{}{:}
\SetKwComment{tcc}{/*}{*/}

\Input{
$\mathcal{G}_m = ( \mathcal{V}, \mathcal{E} )$ and $C$
}
\Output{
$\mathbf{X}$ and $\mathbf{x}_n^{( c )}$
}

\BlankLine

\For{$m \leftarrow 1$ \KwTo $M$}{
Construct adjacency matrix $\mathbf{A}_m$ from $\mathcal{G}_m$.
}

Calculate similarity matrix $\mathbf{S}$ by Equation (\ref{eq:katz_weighted}),
    integrating each $\mathbf{A}_m$ based on its weight $w_m$.

Perform SVD to obtain left singular vectors $\mathbf{U}$ and right singular vectors
    $\mathbf{V}^T$.

Construct multi-robot representation matrix $\mathbf{X}$ from the first $K$
    columns of $\mathbf{U}$ and the first $K$ columns of $\mathbf{V}$.
    %where $K = \frac{D}{2}$.

Randomly select $C$ cluster centroids $\mathbf{\Phi}$.

\Repeat{
Labels have converged
}{
Assign each robot representation $\mathbf{x}_n$ a label $\mathbf{x}_n^{( c )}$
            based on the nearest cluster centroid $\boldsymbol{\phi}_c$.

Update cluster centroids $\mathbf{\Phi}$ based on the labels.
}
\Return $\mathbf{X}$ and $\mathbf{x}_n^{( c )}$

\caption{Graph Embedding for Multi-Robot Div.}
\label{alg:division}
\end{algorithm}

In order to create a lower-dimensional representation than $\mathbf{S}$,
which is necessary when embedding big graphs of a large-scale multi-robot system,
we perform Singular
Value Decomposition (SVD) \cite{golub1970singular} to factor $\mathbf{S}$:
\begin{equation}
\label{eq:svd}
\mathbf{S} = \mathbf{U} \mathbf{\Sigma} \mathbf{V}^T
\end{equation}
The columns of $\mathbf{U}$ are the left singular vectors of
$\mathbf{S}$, the rows of $\mathbf{V}^T$ are the right singular
vectors, and $\mathbf{\Sigma}$ is a diagonal matrix whose entries are
the singular values.
As $\mathbf{S}$ is a square $N \times N$ matrix, all of these are $\in
{\mathcal{R}}^{N \times N}$ when a full SVD is performed.
The columns of $\mathbf{U}$ are also the eigenvectors of
$\mathbf{S}\mathbf{S}^T$, and similarly the rows of $\mathbf{V}^T$
are the eigenvectors of $\mathbf{S}^T \mathbf{S}$.

To further reduce the dimensionality of the representation, given
a desired dimensionality $D$,
we propose to use the
first $K$ left singular vectors and the first $K$ right singular vectors
to approximate $\mathbf{S}$,
where $K = D / 2$, half of the desired dimensionality.
As we only use the first $K$ singular vectors on each side, we can perform
a reduced SVD of $\mathbf{S}$.
This results in $\mathbf{U} \in {\mathcal{R}}^{N \times K}$ and
$\mathbf{V} \in {\mathcal{R}}^{N \times K}$.

Finally, we construct the final representation matrix $\mathbf{X} \in {\mathcal{R}}^{N \times D}$
for all $N$ robots by combining
$\mathbf{U}$ and $\mathbf{V}$:
\begin{equation}
\mathbf{X} = \left( \mathbf{U}, \mathbf{V} \right)
\end{equation}
where each row $\mathbf{x}_n \in \mathcal{R}^{D}$ of $\mathbf{X}$
represents the $n$-th robot, which is then used for division.

\begin{comment}
\begin{figure}[htb]
    \centering
    \includegraphics[width=0.45\textwidth]{media/representation2.png}
    \vspace{-8pt}
    \caption{Illustration of the representation constructed for each
    robot.
    The unified representation matrix $\mathbf{X}$ is formed through the combination
    of $\mathbf{U}$ and $\mathbf{V}$.
    Each robot's representation consists of $K$ values from $\mathbf{U}$
    and $K$ values from $\mathbf{V}$.}
    \label{fig:representation}
\end{figure}
\end{comment}

\subsection{Selection of Multi-Robot Teams}

Given the multimodal graph embedding results,
we propose to perform selection of multi-robot teams
through unsupervised learning (i.e., clustering).
Specifically, given
$\mathbf{X} = \{ \mathbf{x}_1; \dots; \mathbf{x}_N \}$,
and given the number $C$ of
robot teams the human operator needs to accomplish the mission,
the goal is to assign each robot a cluster label
$\mathbf{x}_n^{( c )}$ that indicates it belongs to team $c$,
thus dividing the robotic system into $C$ total teams.

This unsupervised learning problem can be formulated by:
\begin{equation}
\label{eq:cluster}
\underset{\mathbf{x}_n^{( c )}}{\operatorname{arg \, min}} \; \sum_{n=1}^{N}
|| \mathbf{x}_n^{( c )} - \boldsymbol{\phi}_c ||^2
\end{equation}
where $\mathbf{\Phi} = \{ \boldsymbol{\phi}_1, \dots, \boldsymbol{\phi}_C \}$ are
the centroids of all divisions, initially chosen at random.
The formulated optimization problem can be solved iteratively.
First, labels are assigned to each robot $\mathbf{x}_n$.
Then, each centroid $\boldsymbol{\phi}_c$ is recalculated based
on the robots assigned to that cluster.
This process is repeated until the assigned cluster labels $\mathbf{x}_n^{( c )}$
converge.

The algorithmic implementation of both multimodal multi-robot
structure embedding and multi-robot
division is summarized in Algorithm \ref{alg:division}.

\begin{table}[htp]
    \centering
    \caption{Comparison of Accuracy and Silhouette Scores in Scenario I}
    \label{tab:acc_vs_sil}
    \vspace{-5pt}
    \tabcolsep=0.2cm
    \begin{tabular}{|c|c|c|}
    \hline
    Method & Accuracy & Avg. Silhouette Score \\
    \hline\hline
    GN \cite{newman2004finding} (Spatial) & 39.22\% & -0.012 \\
    GN (Connectivity) & 81.37\% & 0.434  \\
    GN (Hierarchy)    & 62.75\% & 0.354  \\
    LLE \cite{roweis2000nonlinear} (Spatial) & 44.12\% & 0.138  \\
    LLE (Connectivity)           & 44.44\% & 0.082  \\
    LLE (Hierarchy)              & 39.22\% & -0.031 \\
    HOPE \cite{ou2016asymmetric} (Spatial) & 85.62\% & 0.455  \\
    HOPE (Connectivity)          & 50.33\% & 0.018  \\
    HOPE (Hierarchy)             & 39.22\% & 0.167  \\
    HOPE (Concatenated)          & 93.14\% & 0.518  \\
    \hline
    Our Approach                 & \textbf{96.73\%} & \textbf{0.680}  \\
    \hline
    \end{tabular}
\end{table}

%%%%
\section{Experiments} \label{sec:results}
%%%%

\begin{table*}[htbp]
    \centering
    \caption{Silhouette Scores for Scenario I}
    \label{tab:army_silhouette}
    \vspace{-8pt}
    \scriptsize
    Columns beginning with `GN' use the Girvan-Newman algorithm \cite{newman2004finding}.
    Columns beginning with `LLE' use Local Linear Embedding \cite{roweis2000nonlinear}.
    Columns beginning with `HOPE' use High-Order Proximity preserved Embedding \cite{ou2016asymmetric}.
    Columns also note whether the embedding represents the \underline{sp}atial graph,
    the \underline{co}nnectivity graph, the \underline{hi}erarchy graph, or a
    \underline{c}oncatenated \underline{c}ombination.
    Our approach integrates all three graphs.
    $k$ indicates the number of teams for that row.
    \\ \vspace{0.10cm}
    \tabcolsep=0.15cm
    \begin{tabular}{|c|c|c|c|c|c|c|c|c|c|c|c|c|}
    \hline
    Formation & $k$ & GN sp. & GN co. & GN hi. & LLE sp. & LLE co. & LLE hi. & HOPE sp. & HOPE co. & HOPE hi. & HOPE cc. & Our Approach \\
    \hline\hline
    Platoon Column & 3 & 0.001 & 0.434 & 0.293 & 0.206 & 0.053 & -0.113 & 0.542 & 0.101 & 0.225 & 0.341 & \textbf{0.763} \\
    %Platoon Column & 4 & 0.035 & 0.435 & 0.347 & 0.225 & 0.070 & -0.104 & 0.435 & -0.146 & 0.156 & 0.554 & \textbf{0.864} \\
    Platoon Column & 5 & 0.065 & 0.460 & 0.411 & 0.241 & 0.079 & -0.170 & 0.347 & -0.128 & 0.008 & 0.572 & \textbf{0.762} \\
    \hline
    Platoon Vee & 3 & -0.099 & 0.569 & 0.164 & 0.206 & 0.129 & -0.064 & 0.569 & 0.019 & 0.324 & 0.447 & \textbf{0.779} \\
    %Platoon Vee & 4 & -0.092 & 0.480 & 0.170 & 0.177 & 0.212 & -0.080 & 0.513 & -0.006 & 0.297 & 0.665 & \textbf{0.879} \\
    Platoon Vee & 5 & -0.089 & 0.317 & 0.283 & 0.200 & 0.246 & -0.217 & 0.313 & 0.034 & 0.144 & 0.634 & \textbf{0.805} \\
    \hline
    Platoon Wedge & 3 & -0.039 & 0.568 & 0.264 & 0.201 & 0.187 & -0.070 & 0.579 & 0.028 & 0.254 & 0.465 & \textbf{0.785} \\
    %Platoon Wedge & 4 & -0.010 & 0.465 & 0.229 & 0.225 & 0.109 & -0.089 & 0.494 & -0.007 & 0.285 & 0.655 & \textbf{0.882} \\
    Platoon Wedge & 5 & 0.004 & 0.304 & 0.362 & 0.232 & 0.276 & -0.149 & 0.367 & 0.035 & 0.125 & 0.626 & \textbf{0.807} \\
    \hline
    Squad Column & 2 & 0.042 & \textbf{0.466} & \textbf{0.466} & 0.068 & -0.027 & 0.107 & \textbf{0.466} & 0.088 & 0.123 & \textbf{0.466} & 0.396 \\
    Squad Column & 3 & -0.012 & 0.323 & 0.467 & 0.028 & -0.014 & 0.064 & 0.483 & 0.273 & 0.114 & 0.467 & \textbf{0.604} \\
    \hline
    Squad File & 2 & 0.036 & \textbf{0.423} & \textbf{0.423} & -0.024 & 0.000 & 0.114 & \textbf{0.423} & 0.077 & 0.102 & \textbf{0.423} & 0.325 \\
    Squad File & 3 & -0.026 & 0.454 & 0.429 & 0.023 & -0.031 & 0.083 & 0.283 & -0.044 & 0.081 & 0.454 & \textbf{0.530} \\
    \hline
    Squad Line & 2 & 0.054 & \textbf{0.494} & \textbf{0.494} & 0.049 & -0.007 & 0.147 & \textbf{0.494} & -0.015 & 0.147 & \textbf{0.494} & 0.420 \\
    Squad Line & 3 & -0.054 & 0.328 & 0.514 & 0.014 & -0.047 & 0.075 & 0.514 & -0.044 & 0.124 & 0.514 & \textbf{0.651} \\
    %\hline
    %Average & - & -0.012 & 0.434 & 0.354 & 0.138 & 0.082 & -0.031 & 0.455 & 0.018 & 0.167 & 0.518 & \textbf{0.680} \\
    \hline
    \end{tabular}
\end{table*}

To evaluate the performance of the proposed multimodal
graph embedding approach for selection of multi-robot teams,
we conduct experiments using robotic systems at multiple scales.
We evaluate the accuracy of multi-robot team selection and
score the quality of the robot clusters.
We also compare with previous graph-based division techniques.

\subsection{Experimental Setup}

Our approach was experimentally evaluated and validated in
two different scenarios.
First, we evaluate on \emph{Scenario I: Expert-defined Team Formations.}
This experiment is based on a set of expert-defined graphs that
represent relationships in field operations teams at two scales.
Second, we evaluate on
\emph{Scenario II: Simulated Large-Scale Multi-Robot Systems.}
We generate large-scale multi-robot systems
to demonstrate that our method scales beyond the
size of small groups.
These simulated systems do not have ground truth teams, and instead
our identified teams are scored by a clustering metric.
We demonstrate our approach can be deployed in real-world robotic applications
by deploying some of these simulated robotic systems as physical TurtleBot robots.

\begin{figure}[tb]
    \centering
    \includegraphics[width=0.45\textwidth]{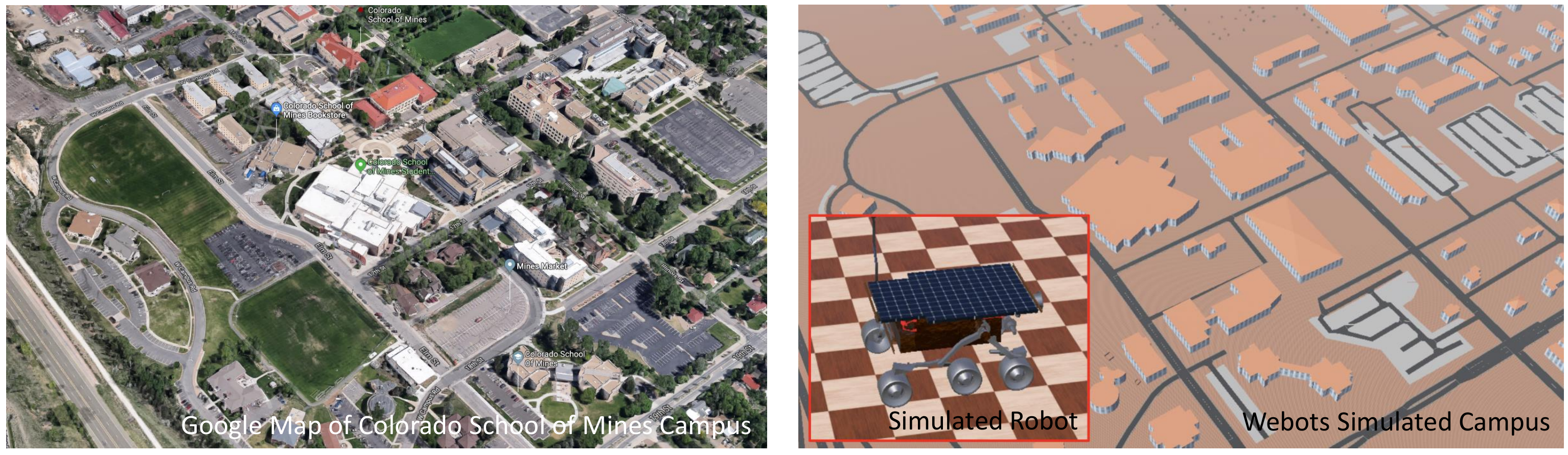}
    \vspace{-8pt}
    \caption{Experimental setups in a robot-assisted urban search and
    rescue scenario used to evaluate our approach in Scenario I and
    Scenario II.
    The left figure depicts the Google satellite map of the real
    campus environment of the Colorado School of Mines.
    The right figure illustrates the simulated environment and
    a simulated robot in the Webots robot simulator.}
    \label{fig:webots_env}
\end{figure}

\begin{figure}
    \centering
    \subfigure[Squad Column]{
        \label{fig:sc}
        \centering
        \includegraphics[height=0.95in]{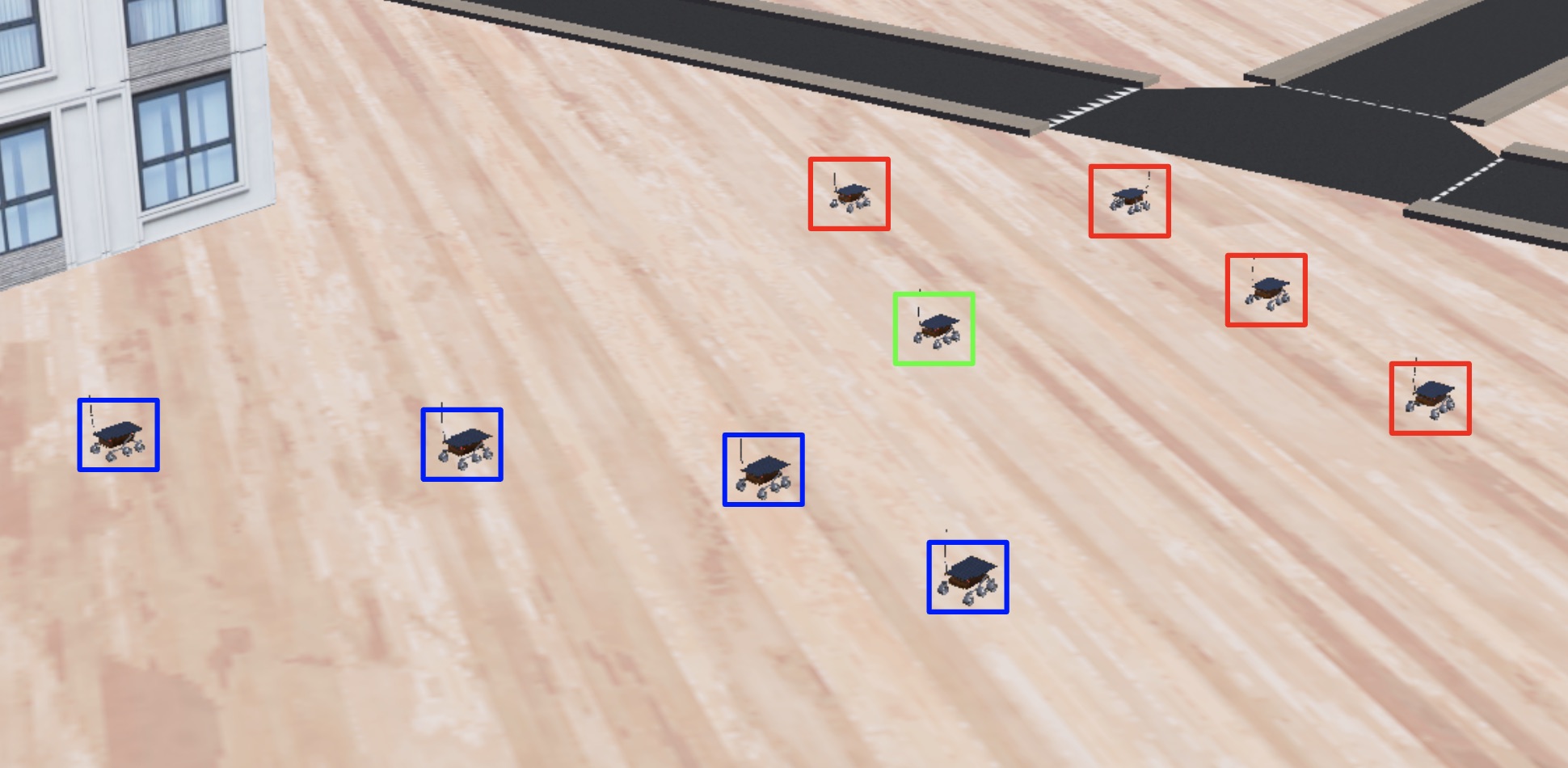}
    }%
    \subfigure[Platoon Column]{
        \label{fig:pc}
        \centering
        \includegraphics[height=0.95in]{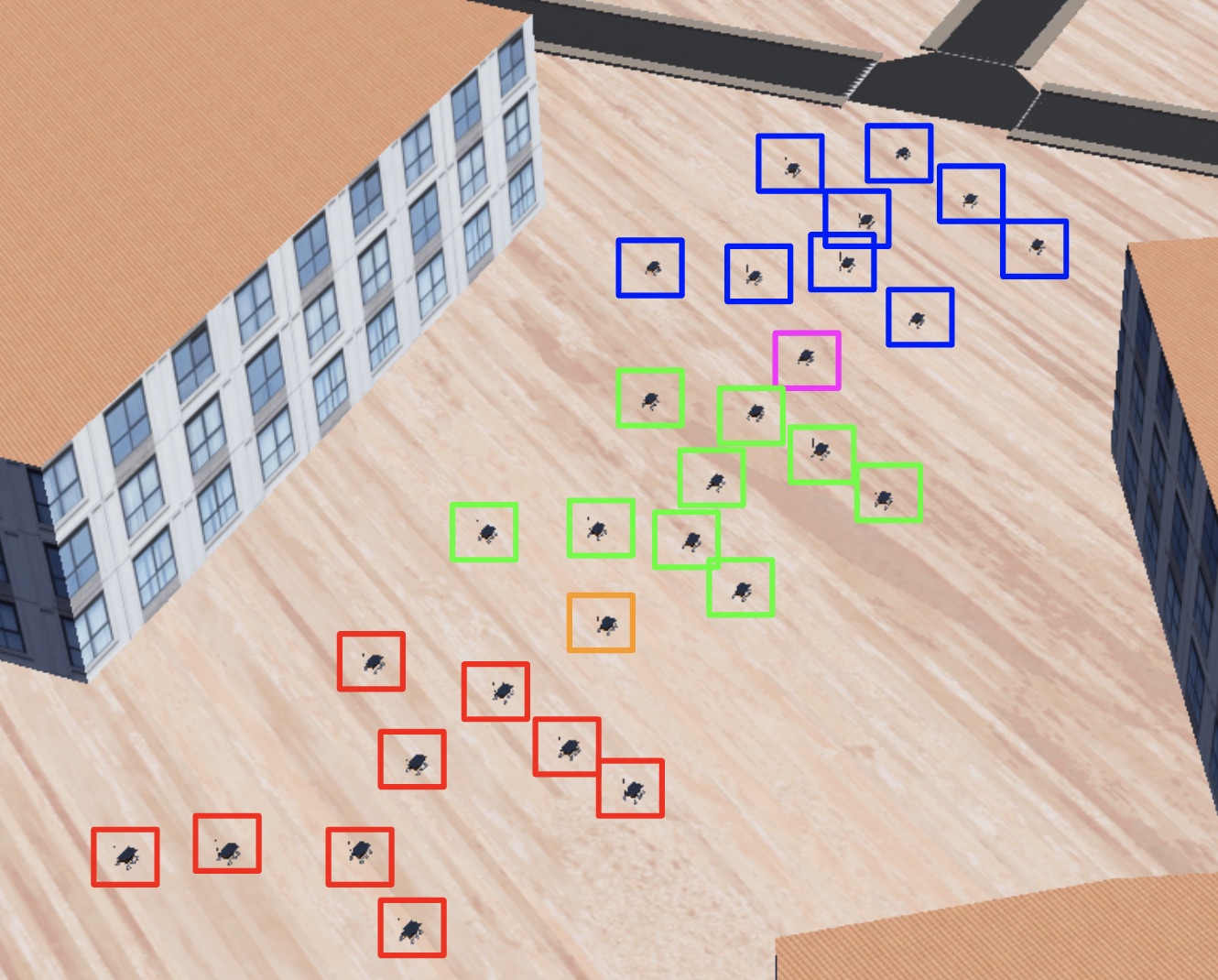}
    }
    \vspace{-8pt}
    \caption{Qualitative results of multi-robot team selection in Scenario I.
    In this scenario, the multi-robot system is
    approaching an intersection where division
    into teams may be necessary to effectively search down branching
    roads or to secure the intersection.
    Figure \ref{fig:sc} displays the \emph{squad column}, with colored
    bounding boxes identifying the correct division into three teams
    \cite{fm3_21_9}.
    Figure \ref{fig:pc} displays the \emph{platoon column}, displayed
    with the correct division into five teams.}
    \label{fig:formations}
\end{figure}

Simulations of the expert-defined teams and the large-scale multi-robot systems were performed
with the Webots robot simulator \cite{webots} in an application of area exploration.
We created a simulated environment of the Colorado School of Mines campus for the experiments.
Figure \ref{fig:webots_env} displays the real campus environment, its simulated Webots map and the simulated robot used.
In the experiments on all scenarios,
three types of graphs are used as input information modalities,
describing spatial positions, known
communication connectivity, and a set hierarchy.
For quantitative evaluation,
we present the clustering accuracy when dividing the expert-defined teams
and utilize
silhouette scores \cite{rousseeuw1987silhouettes} for all robotic divisions.
Silhouette scores rate the quality of a clustering, with values closer
to 1 being better and values closer to -1 being worse.

To validate the superior performance of our approach,
we compare it with previous methods:
(1) \emph{Girvan-Newman} (GN) \cite{newman2004finding}, a single-modality graph community finding algorithm;
(2) \emph{Local Linear Embedding} (LLE)
\cite{roweis2000nonlinear}, a single-modality graph embedding approach;
(3)  \emph{High-Order Proximity preserved Embedding} (HOPE)
\cite{ou2016asymmetric}, a single-modality graph embedding approach;
and (4) \emph{Concatenated Combination HOPE}, which applies the HOPE graph embedding on each individual graph and
concatenates these vectors from three separate graphs into a single representation.

\begin{figure}
    \centering
    \includegraphics[width=0.44\textwidth]{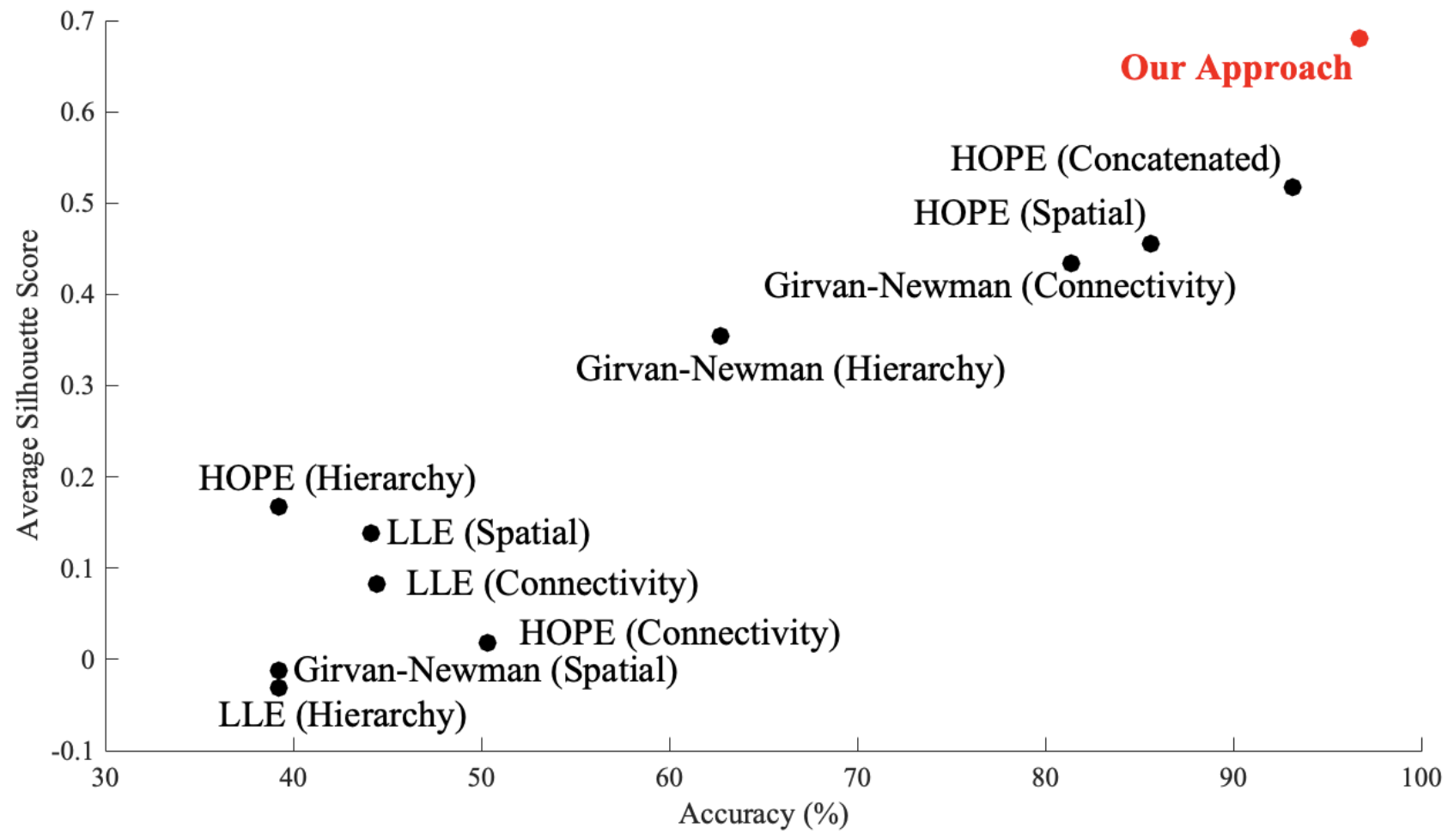}
    \vspace{-8pt}
    \caption{Quantitative experimental results on the relationship of accuracy and silhouette
    score, which shows a linear relationship between higher
    accuracy and higher silhouette scores.
    Our multimodal graph embedding approach achieves both the highest accuracy and the highest silhouette score, which outperforms the previous methods.}
    \label{fig:acc_vs_sil}
\end{figure}

\subsection{Results on Expert-Defined Team Formations (Scenario I)}

We first evaluate our approach on the
expert-defined team formations known as
the \emph{platoon column}, \emph{platoon wedge}, and \emph{platoon vee}
and the \emph{squad column}, \emph{squad file}, and \emph{squad line},
based on the field operations teaming protocol in \cite{fm3_21_9}.
This protocol contains correct, expert-defined
sub-divisions for these formations.
Platoon formations incorporate three squads and two separate
leadership agents.
Squad formations incorporate two teams and one separate leadership
agent.
Figure \ref{fig:formations} displays the \emph{squad column}
and \emph{platoon column} in the Webots simulator, with correct
sub-divisions labeled.

\begin{figure*}
    \centering
    \subfigure[Squad Formations]{
        \label{fig:noise_squad}
        \centering
        \includegraphics[width=0.231\textwidth]{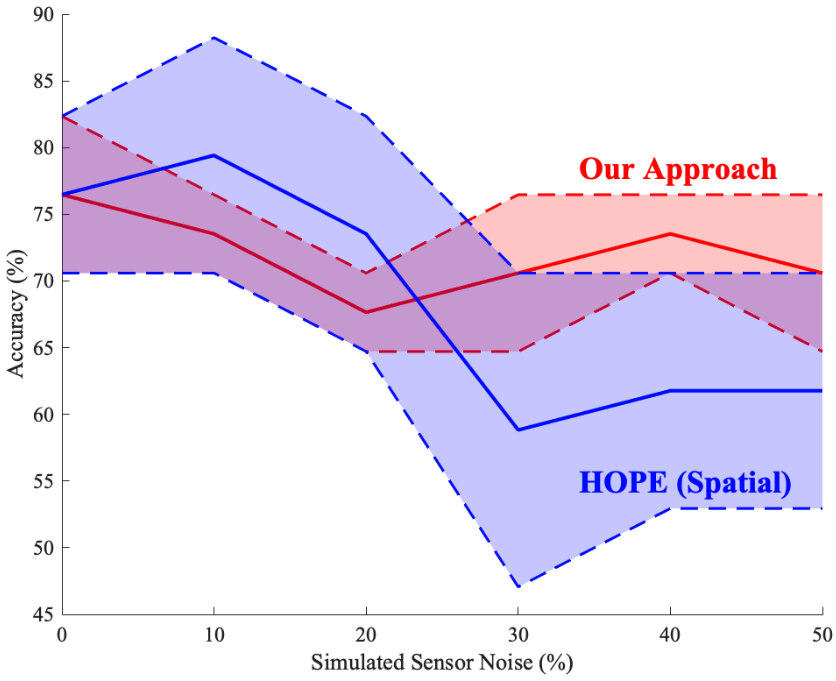}
    }%
    \subfigure[Platoon Formations]{
        \label{fig:noise_platoon}
        \centering
        \includegraphics[width=0.232\textwidth]{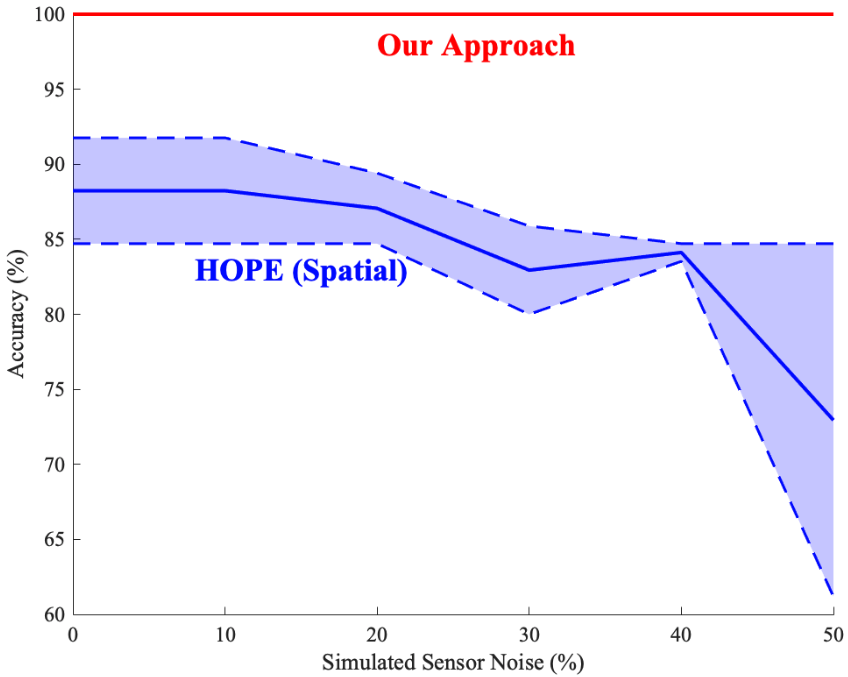}
    }%
    \subfigure[Squad Formations]{
        \label{fig:noise_sil_squad}
        \centering
        \includegraphics[width=0.232\textwidth]{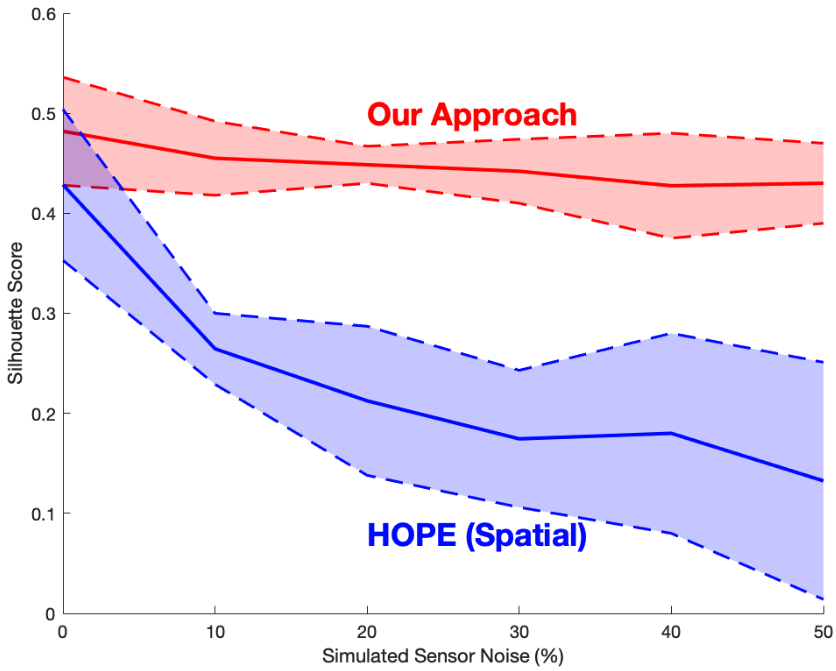}
    }%
    \subfigure[Platoon Formations]{
        \label{fig:noise_sil_platoon}
        \centering
        \includegraphics[width=0.232\textwidth]{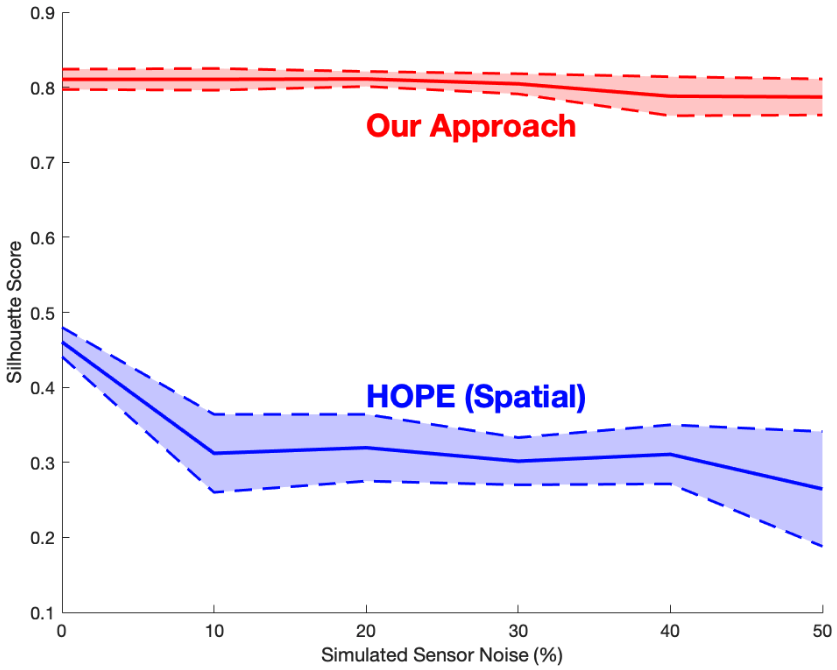}
    }
    \vspace{-8pt}
    \caption{Experimental results showing the effect
    of simulated sensor noise when determining the team assignments of robots
    on both clustering accuracy and silhouette score.
    Figure \ref{fig:noise_squad} and Figure \ref{fig:noise_platoon}
    show the effects on clustering accuracy on the squad and
    platoon formations respectively, compared to the HOPE embedding
    method \cite{ou2016asymmetric}.
    Figure \ref{fig:noise_sil_squad} and Figure \ref{fig:noise_sil_platoon}
    show the effects on the silhouette score metric for the
    squad and platoon formations, again compared to to the HOPE
    embedding method based solely on spatial information.
    }
    \label{fig:noise}
\end{figure*}

The spatial relationships, communication
capabilities, and structured hierarchy of these formations are
defined by the field operations teaming protocol.
We encode each of these relationship modalities as a separate graph in order
to compare our approach to previous methods.
Table \ref{tab:acc_vs_sil} reports the clustering accuracy for our
approach versus these baseline approaches.
Out of a possible 306 agents in the six different formations, our
approach clusters $96.73\%$ of the agents correctly.
The second best is the concatenated combination of HOPE embeddings,
clustering $93.14\%$ of agents correctly, showing that extending
existing graph embedding methods to leverage multiple information
modalities can significantly improve their performance from using
a single modality.

Table \ref{tab:army_silhouette} displays the silhouette scores for
each sub-division of each formation.
Our approach achieves the highest score, with an average silhouette
score of $0.680$.
Again, second highest is the concatenated combination of HOPE
embeddings, scoring $0.518$.
We note that our approach achieves its best results on the platoon
formations, which contain over three times as many agents as the squad formations.
This suggests that our approach's performance will extend to
larger multi-robot systems.
There also exists a linear relationship between clustering accuracy and
silhouette scores, visualized in Figure \ref{fig:acc_vs_sil}.
This validates the silhouette score as a metric for dividing the simulated
robotic systems in Scenario II, where the ground truth divisions are unknown.

For these expert-defined teams, we also evaluated the effect of simulated
sensor noise, where robots no longer have exact knowledge of where
their robotic teammates are.
Figure \ref{fig:noise} shows both the effect on clustering accuracy of adding up to $50\%$ error
into the distances between robots and the effect
of this noise on silhouette score.
Our approach, able to utilize information from other modalities,
is robust to this error, declining in accuracy only slightly for
the smaller squad formations and maintaining $100\%$ accuracy
on the larger platoon formations.
Our approach is similarly consistent in maintaining a high silhouette
score despite the simulated sensor noise.
The HOPE embedding, based on the spatial graph, is affected significantly.
On the squad formations, this method declines from an average
accuracy of $76.47\%$ with no sensor noise to $61.77\%$ with $50\%$ noise.
A similar decline in performance occurs with the platoon formations,
from $88.24\%$ to $72.95\%$ accuracy.
The HOPE embedding also produces clusters with lower silhouette scores
as the simulated sensor noise increases.

\begin{figure}[]
    \centering
    \includegraphics[width=0.42\textwidth]{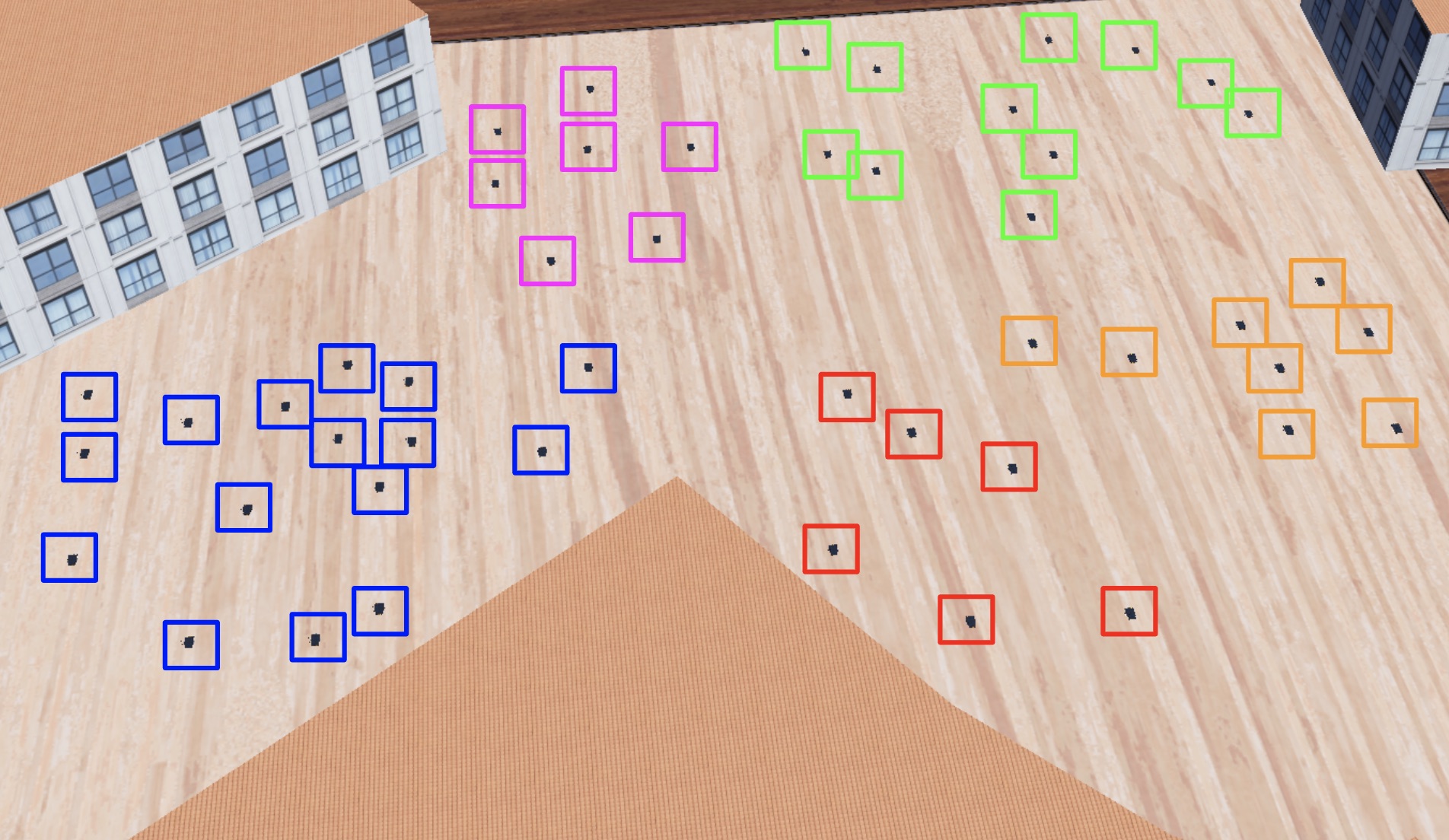}
    \vspace{-8pt}
    \caption{Qualitative results on multi-robot team selection over a
    system of 50 simulated robots, with bounding boxes
    indicating the five teams identified by our approach.}
    \label{fig:generated}
    \vspace{-5pt}
\end{figure}

\begin{table}[]
    \centering
    \caption{Silhouette Scores for Scenario II}
    \label{tab:generated_silhouette}
    \vspace{-8pt}
    \tabcolsep=0.02cm
    \begin{tabular}{|c|c|c|c|c|c|c|c|c|c|}
    \hline
    Bots & $k$ & LLE s. & LLE c. & LLE h. & HOPE s. & HOPE c. & HOPE h. & HOPE cc. & Ours \\
    \hline\hline
    10 & 3 & 0.19 & 0.17 & 0.02 & 0.34 & 0.44 & 0.19 & 0.35 & \textbf{0.68} \\
    10 & 4 & 0.22 & 0.15 & 0.02 & 0.34 & 0.39 & 0.17 & 0.39 & \textbf{0.66} \\
    10 & 5 & 0.23 & 0.10 & 0.03 & 0.32 & 0.37 & 0.07 & 0.40 & \textbf{0.58} \\
    \hline
    20 & 3 & 0.34 & 0.30 & -0.02 & 0.44 & 0.35 & 0.27 & 0.44 & \textbf{0.70} \\
    20 & 4 & 0.41 & 0.36 & -0.01 & 0.45 & 0.44 & 0.21 & 0.49 & \textbf{0.68} \\
    20 & 5 & 0.45 & 0.32 & -0.01 & 0.45 & 0.34 & 0.15 & 0.52 & \textbf{0.65} \\
    \hline
    30 & 3 & 0.36 & 0.23 & -0.01 & 0.43 & 0.19 & 0.28 & 0.44 & \textbf{0.69} \\
    30 & 4 & 0.45 & 0.29 & -0.02 & 0.46 & 0.30 & 0.22 & 0.48 & \textbf{0.70} \\
    30 & 5 & 0.47 & 0.27 & -0.01 & 0.47 & 0.22 & 0.19 & 0.51 & \textbf{0.67} \\
    \hline
    40 & 3 & 0.36 & 0.17 & -0.00 & 0.44 & 0.10 & 0.27 & 0.44 & \textbf{0.68} \\
    40 & 4 & 0.45 & 0.20 & -0.03 & 0.47 & 0.12 & 0.23 & 0.48 & \textbf{0.71} \\
    40 & 5 & 0.49 & 0.20 & -0.02 & 0.46 & 0.11 & 0.20 & 0.49 & \textbf{0.66} \\
    \hline
    50 & 3 & 0.36 & 0.12 & -0.02 & 0.42 & 0.03 & 0.27 & 0.42 & \textbf{0.70} \\
    50 & 4 & 0.45 & 0.13 & -0.03 & 0.46 & 0.02 & 0.23 & 0.47 & \textbf{0.71} \\
    50 & 5 & 0.48 & 0.15 & -0.03 & 0.46 & 0.03 & 0.21 & 0.49 & \textbf{0.66} \\
    %\hline
    %Avg. & - & 0.38 & 0.21 & -0.01 & 0.43 & 0.23 & 0.21 & 0.45 & \textbf{0.68} \\
    \hline
    \end{tabular}
\end{table}

\subsection{Results on Simulated Large-scale Multi-Robot Systems (Scenario II)}

To evaluate the effectiveness of our approach on a larger scale,
we simulated larger multi-robot systems in Webots.
These large-scale systems have the same relationship modalities
as the expert-defined teams.
Spatial relationships are calculated from their simulated physical
locations, communication capabilities are based on each simulated
robot being capable of short-range, line-of-sight communication
to nearby robots, and hierarchical relationships were defined
based on heuristics from field operations teams (i.e., robots near
the center of the overall system are higher in the hierarchy than
robots on the perimeter).
Figure \ref{fig:generated} displays a simulated system of 50
robots in Webots.

We repeatedly generated these large multi-robot systems
and performed team selection based on each evaluated approach.
Table \ref{tab:generated_silhouette} shows results for simulated
systems of 10, 20, 30, 40 and 50 robots divided into three, four and five teams, with
comparisons to existing graph embedding approaches.
Each combination of size and number of clusters was run 100
times.
As seen in the experimental results in Table \ref{tab:generated_silhouette},
our approach achieves the highest silhouette score of $0.676$.
This is consistent with our score of $0.680$ on the expert-defined teams, where
our approach achieved the highest clustering accuracy, suggesting
that our identified sub-divisions of large-scale simulated multi-robot systems
outperform the divisions identified by other methods.

\begin{figure}[tbh]
    \centering
    \subfigure[]{
        \centering
        \includegraphics[height=0.58in]{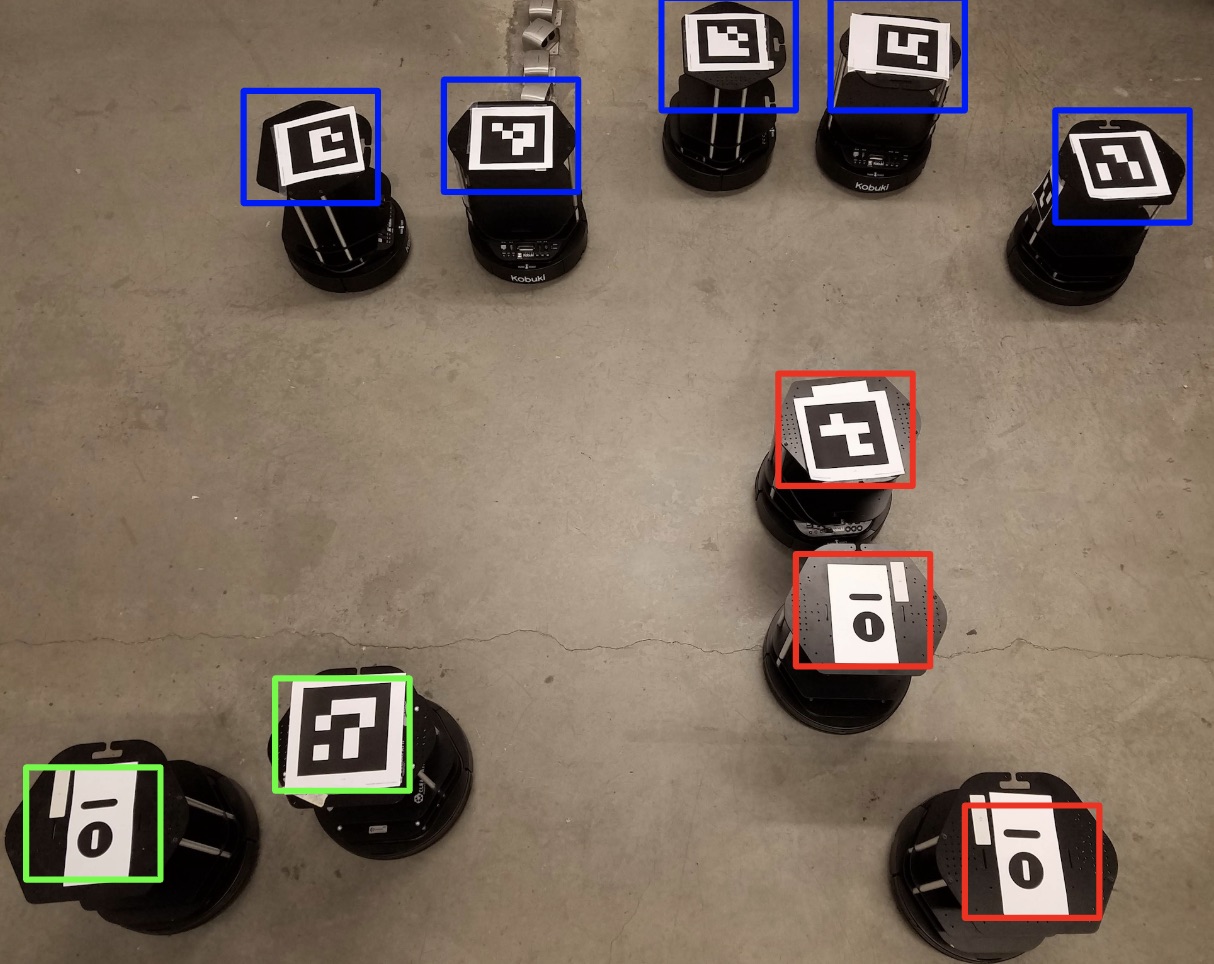}
    }%
    \subfigure[]{
        \centering
        \includegraphics[height=0.58in]{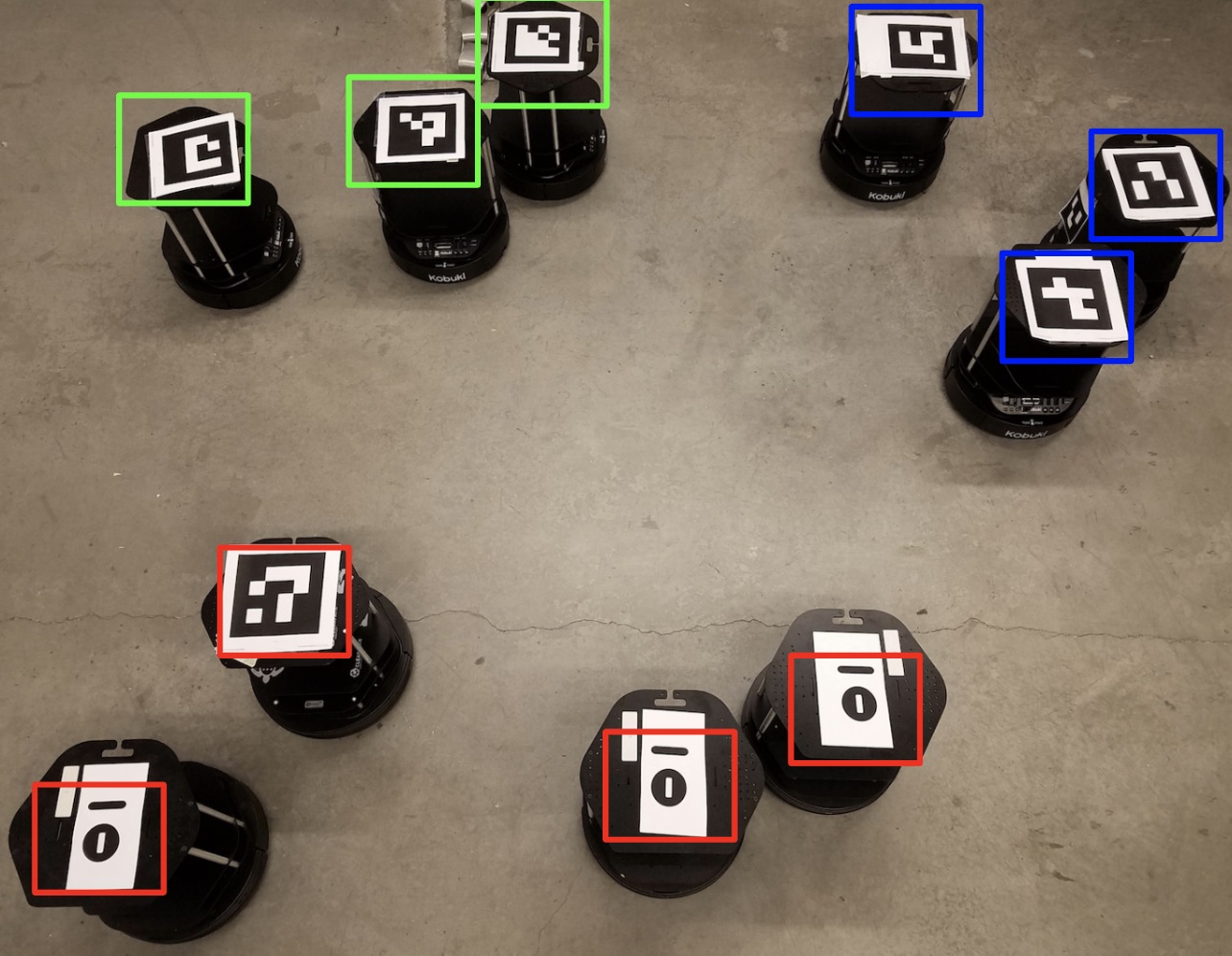}
    }%
    \subfigure[]{
        \centering
        \includegraphics[height=0.58in]{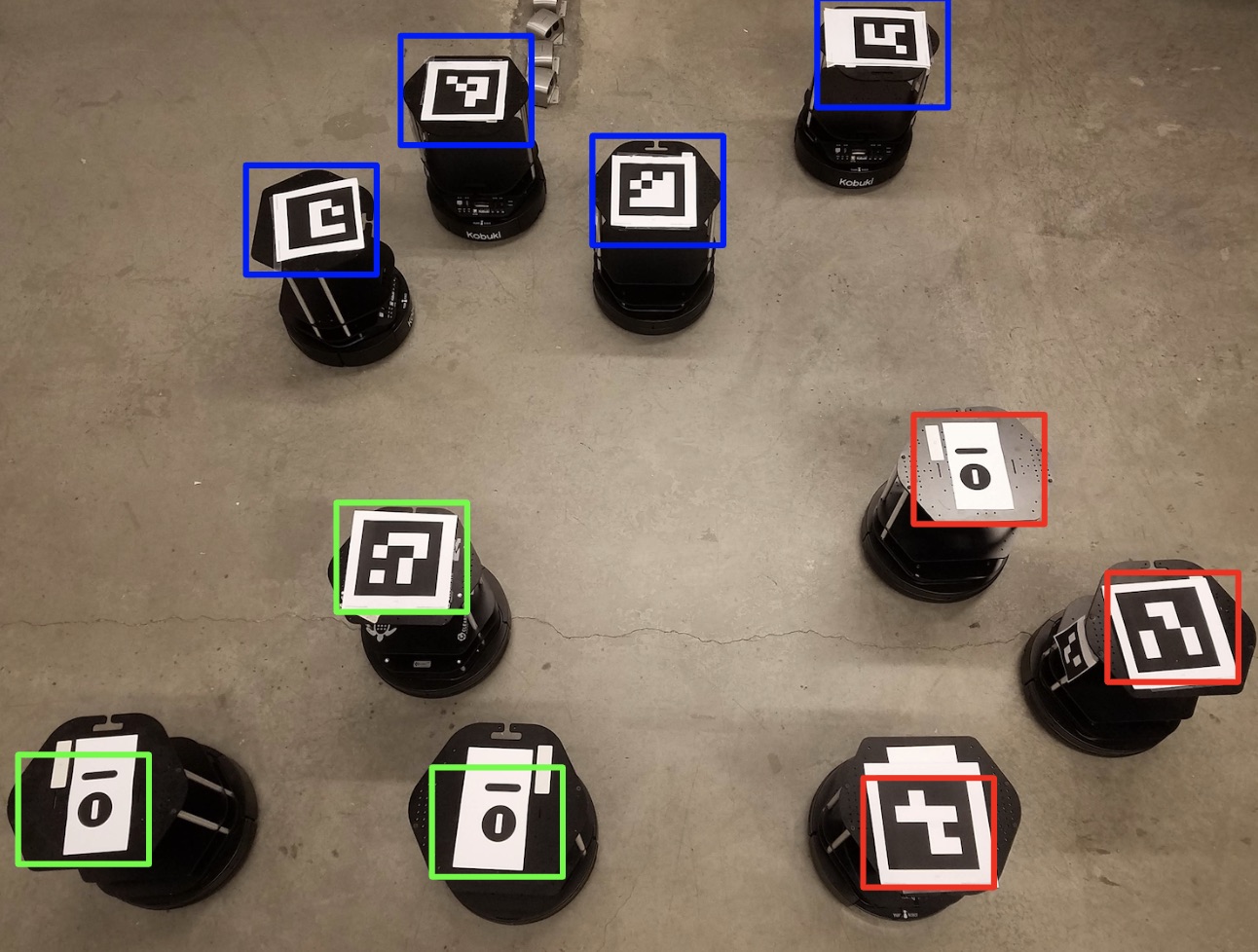}
    }%
    \subfigure[]{
        \centering
        \includegraphics[height=0.58in]{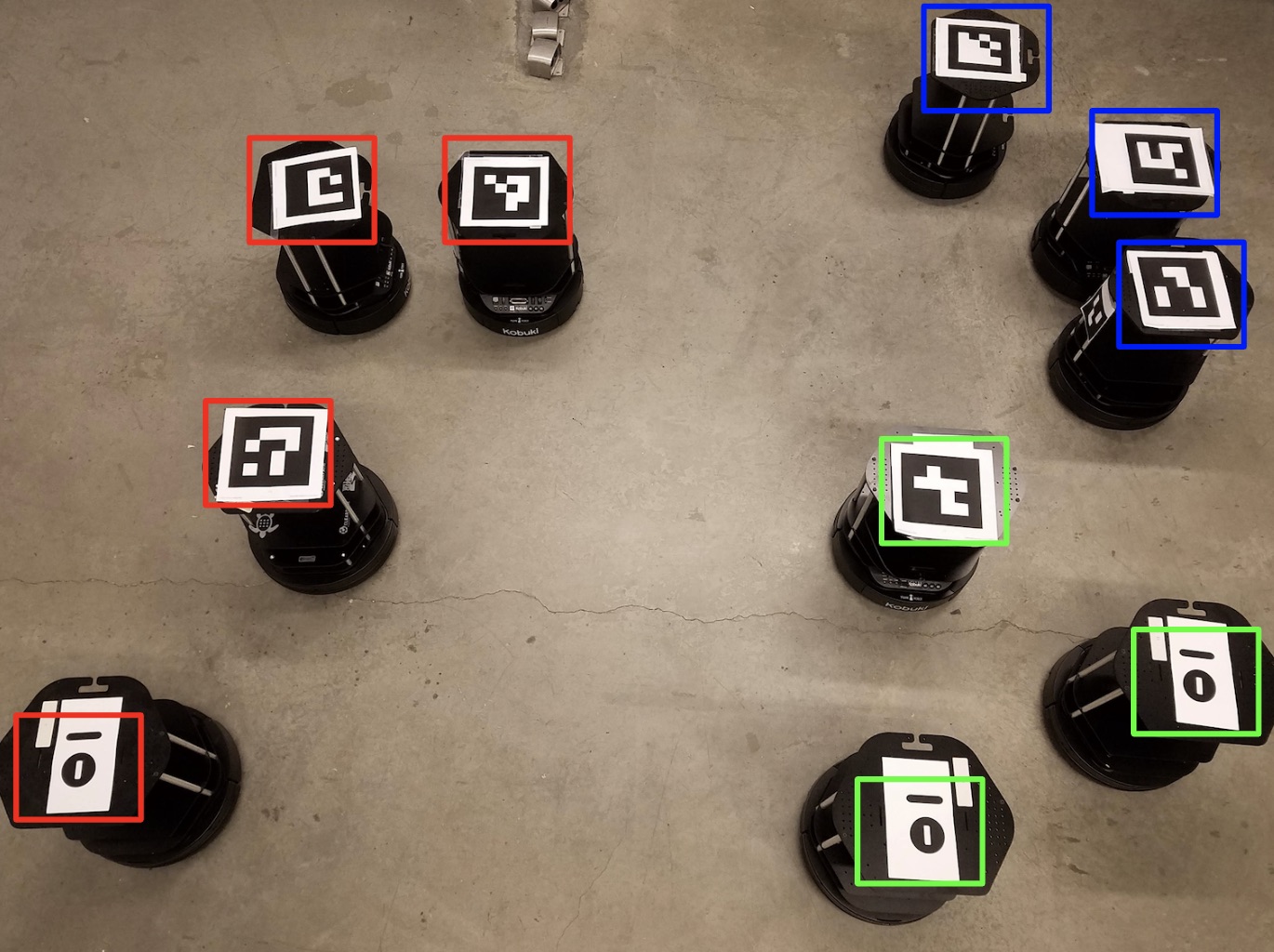}
    }
    \vspace{-12pt}
    \caption{Overhead views of four of the Turtlebot multi-robot systems.
    Bounding boxes indicate which team each Turtlebot is assigned to.
    }\label{fig:realworld}
\end{figure}

To demonstrate the ability and potential of our approach to work on physical
robots in the real world, we implemented systems of
10 physical Turtlebot
robots in a proof-of-concept case study.
We labeled these 10 Turtlebots with tracking tags, and utilized an overhead camera
to track the robots and
construct their spatial relationship graph from their positions.
Communication and hierarchy graphs were defined as they were in the Webots
simulator.
%Using these three graphs as the input information modalities,
We applied our proposed multimodal graph embedding approach to identify teams within
the Turtlebot system, with the objective to divide the physical robot system
into three teams.
Figure \ref{fig:realworld} illustrates four executions of this as a case study, with bounding boxes identifying the positions and team labelings.
%within the Turtlebot system.

%%%%
\section{Conclusion} \label{sec:conclusion}
%%%%

In this paper, we introduce a novel approach to representing multi-robot
structure in order to divide a multi-robot
system into teams.
The proposed approach represents robotic relationships as directed graphs,
and constructs a vector representation from multiple graphs through
a novel multimodal graph embedding method.
Our approach is able to integrate multiple information modalities to describe
and divide a multi-robot system,
allowing our approach to create teams that are more comprehensive and effective.
We demonstrate the effectiveness of our approach over expert-defined team
formations
%showing that it outperforms both graph division approaches and
%other graph embedding techniques.
%Additionally,
and evaluate our approach on
large-scale simulated multi-robot systems and on physical robot teams.
%to demonstrate the superior performance of our approach and
%its potential to be deployed on real-world robot systems.

%%%% BIBLIOGRAPHY
% \clearpage
\bibliographystyle{ieeetr}
\bibliography{references}

\end{document}